\documentclass[twoside]{article}

\usepackage[accepted]{aistats2021}
\usepackage[round]{natbib}

\usepackage{times}
\usepackage{soul}
\usepackage{url}
\usepackage[hidelinks]{hyperref}
\usepackage[utf8]{inputenc}
\usepackage[small]{caption}
\usepackage{graphicx}
\usepackage{amsmath}
\usepackage{amsthm}
\usepackage{booktabs}
\usepackage{amsmath}
\usepackage{amssymb}
\usepackage{amsfonts}
\usepackage{amsthm}
\usepackage{bbm}
\usepackage{xparse}
\usepackage{mathtools}
\usepackage{amsmath,bm}
\usepackage{algorithm}
\usepackage{xcolor}
\usepackage[noend]{algpseudocode}
\def\prox{\mathop{\rm prox}\nolimits}

\DeclareMathOperator{\E}{\mathbb{E}}
\usepackage{subcaption}
\newtheorem{lemma}{Lemma}
\usepackage{booktabs}
\makeatletter
\usepackage{bbm}
\def\BState{\State\hskip-\ALG@thistlm}
\makeatother
\usepackage{bbm}
\usepackage{xpatch}
\DeclarePairedDelimiter{\ceil}{\lceil}{\rceil}

\DeclareMathOperator*{\argmin}{arg\,min}
\makeatletter 

\newtheorem{theorem}{Theorem}
\usepackage{xpatch}
\usepackage{enumitem}

\setitemize{fullwidth}

\makeatletter
\xpatchcmd{\algorithmic}
  {\ALG@tlm\z@}{\leftmargin\z@\ALG@tlm\z@}
  {}{}
\makeatother

\begin{document}



\twocolumn[

\aistatstitle{Variance Reduced Stochastic Proximal Algorithm for AUC Maximization}

\aistatsauthor{ Soham Dan \And Dushyant Sahoo }

\aistatsaddress{ University of Pennsylvania \And university of Pennsylvania } ]

\begin{abstract}
 Stochastic Gradient Descent has been widely studied with classification accuracy as a performance measure. However, these stochastic algorithms cannot be directly used when non-decomposable pairwise performance measures are used, such as Area under the ROC curve (AUC), a standard performance metric when the classes are imbalanced. There have been several algorithms proposed for optimizing AUC as a performance metric, one of the recent being a Stochastic Proximal Gradient Algorithm (SPAM). However, the downside of stochastic gradient descent is that it suffers from high variance leading to slower convergence. Several variance reduced methods have been proposed with faster convergence guarantees than vanilla stochastic gradient descent to combat this issue. Again, these variance reduced methods are not directly applicable when non-decomposable performance measures are used. In this paper, we develop a Variance Reduced Stochastic Proximal algorithm for AUC Maximization (\textsc{VRSPAM}) and perform a theoretical analysis and empirical analysis. We show that our algorithm converges faster than SPAM, the previous state-of-the-art for the AUC maximization problem.
\end{abstract}

\section{Introduction}
Several applications face an issue of class imbalance---the case where one of the classes occurs much more frequently than the other class \citep{elkan2001foundations}. A concrete example is a medical diagnosis for a rare disease where far fewer instances from the disease class are observed than the healthy class. Traditional classification accuracy is not an appropriate performance metric in this setting, as predicting the majority class will give a high classification accuracy. To overcome this drawback, the Area under the ROC curve (AUC) \citep{fawcett2006introduction} is used as a standard metric for quantifying the performance of a binary classifier in this setting. AUC measures the ability of a family of classifiers to correctly rank an example from the positive class with respect to a randomly selected example from the negative class.

Several algorithms have been proposed for AUC maximization in the batch setting, where all the training data is assumed to be available at the beginning \citep{herschtal2004optimising,zhang2012smoothing}. However, this assumption is unrealistic in several cases, especially for streaming data analysis, where examples are observed one at a time. Several online algorithms have been proposed for the usual classification accuracy metric for such a streaming setting where the per iteration complexity is low \citep{shalev2012online,orabona2014simultaneous}. Despite these several studies on online algorithms for classification accuracy,  the case of maximizing AUC as a performance measure has been looked at only recently \citep{zhao2011online,kar2013generalization}. The main challenge for optimizing the AUC metric in the online setting is the metric's pairwise nature compared to classification accuracy, which decomposes over individual instances. In the AUC maximization framework, in each step, the algorithm needs to pair the current datapoint with all previously seen datapoints leading to $\mathcal{O}(td)$ space and time complexity at step $t$, where the dimension of the instance space is $d$. The problem was not alleviated by the technique of buffering \citep{zhao2011online,kar2013generalization} since good generalization performance depends on maintaining a large buffer. 

%Recently, \cite{palaniappan2016stochastic} provided a primal-dual algorithm by extending stochastic variance reduced algorithms (SVRG, SAGA) to handle non-decomposable losses or regularizers (in the form of convex-concave saddle point problem)and thereby provided convergence rate $\mathcal{O}(\frac{1}{t})$. Although this can be applied to AUC optimization with the least-squared loss, their algorithm needs to assume strong convexity of both the primal and dual variables. Their algorithm also has expensive per-iteration complexity $\mathcal{O}(n+d)$ where $n$ is the number of data points, and $d$ is the dimension.

From an optimization perspective, the AUC metric is non-convex and thus hard to optimize. Instead, it is attractive to optimize the convex surrogate, which is consistent, such as the pairwise squared surrogate \citep{agarwal2013surrogate,narasimhan2017support,gao2015consistency}.
Recently, \citet{ying2016stochastic} reformulated the pairwise squared loss surrogate of AUC as a saddle point problem and gave an algorithm that has a convergence rate of  $\mathcal{O}(\frac{1}{\sqrt{t}})$. However, they only consider smooth regularization (penalty) terms such as Frobenius norm. Further, their convergence rate is sub-optimal to what stochastic gradient descent (SGD) achieves with classification accuracy as a performance measure $\mathcal{O}(\frac{1}{t})$. \citet{natole2018stochastic} improves on this with a stochastic proximal algorithm for AUC maximization, which under assumptions of strong convexity can achieve a convergence rate of $\mathcal{O}(\frac{\log t}{t})$ and has per iteration complexity of $\mathcal{O}(d)$ i.e., one datapoint and applies to general, non-smooth regularization terms. 

Although \citet{natole2018stochastic} improves convergence for surrogate-AUC maximization, it still suffers from a high variance of the gradient in each iteration. Due to the large variance in random sampling, the stochastic gradient algorithm wastes time bouncing around, leading to worse performance and slower sub-linear convergence rate of $\mathcal{O}(\frac{1}{t})$ (even if we ignore the $\log(t)$ term). Thus, we have a low per iteration complexity for the stochastic algorithm, but slow convergence contrasted with high per iteration complexity and fast convergence for full gradient descent. Thus, it might take longer to get a good approximation of the optimization problem's solution if we employ the algorithm proposed by \citet{natole2018stochastic}.

In the more straightforward context of classification accuracy, techniques to reduce the variance of SGD have been proposed---SAG \citep{roux2012stochastic}, SDCA \citep{shalev2013stochastic}, SVRG \citep{johnson2013accelerating}. While SAG and SDCA require the storage of all the gradients and dual variables, respectively, for complex models, SVRG enjoys the same fast convergence rates as SDCA and SAG but has a much simpler analysis and does not require storage of gradients. This allows SVRG to be applicable in complex problems where the storage of all gradients would be infeasible. 

Several works have explored ways to tackle the presence of a regularizer term and the average of several smooth component function terms in SVRG. Two simple strategies are to use the Proximal Full Gradient and the Proximal Stochastic Gradient method. While the Proximal Stochastic Gradient is much faster since it computes only the gradient of a single component function per iteration, it convergences much slower than the Proximal Full Gradient method. The proximal gradient methods can be viewed as a particular case of splitting methods \citep{bauschke2011convex,beck2008fast}. However, both the proximal methods do not fully exploit the problem structure. Proximal SVRG \citep{xiao2014proximal} is an extension of the SVRG \citep{johnson2013accelerating} technique and can be used whenever the objective function is composed of two terms- the first term is an average of smooth functions (decomposes across the individual instances), and the second term admits a simple proximal mapping. Prox-SVRG needs far fewer iterations to achieve the same approximation ratio than the proximal full and stochastic gradient descent methods. However, an important gap has not been addressed yet - existing techniques that guarantee faster convergence by controlling the variance are not directly applicable to non-decomposable pairwise loss functions as in surrogate-AUC optimization, this is the gap that we close in this paper. 

In this paper, we present Variance Reduced Stochastic Proximal algorithm for AUC Maximization (VRSPAM). VRSPAM builds upon previous work for surrogate-AUC maximization by using the SVRG algorithm. We provide theoretical analysis for the VRSPAM algorithm showing that it achieves a linear convergence rate with a fixed step size (better than SPAM \citep{natole2018stochastic}, which has a sub-linear convergence rate and constantly decreasing step size). Also, the theoretical analysis provided in the paper is more straightforward than the analysis of SPAM. We perform numerical experiments to show that the VRSPAM algorithm converges faster than SPAM.

\section{AUC formulation}
The AUC score associated with a linear scoring function $g(x) = \mathbf{w}^Tx$, is defined as the probability that the score of a randomly chosen positive example is higher than a randomly chosen negative example \citep{hanley1982meaning,clemenccon2008ranking} and is denoted by $\text{AUC}(\mathbf{w})$. If $z=(x,y)$ and $z'=(x',y')$ are drawn independently from an unknown distribution $\mathcal{Z}=\mathcal{X}\times \mathcal{Y}$, then
\begin{align*}
\text{AUC}(\mathbf{w})&=Pr(\mathbf{w}^Tx \geq \mathbf{w}^Tx'| y=1,y'=-1)\\ 
&= \E[\mathbb{I}_{\mathbf{w}^T(x-x')\geq 0}|y=1,y'=-1]
\end{align*}
Since $\text{AUC}(\mathbf{w})$ in the above form is not convex because of the 0-1 loss, it is a common practice to replace this by a convex surrogate loss. In this paper, we focus on the least square loss which is known to be consistent (maximizing the surrogate function also maximizes the AUC). Let $f(\mathbf{w}) = p(1-p)\E[(1-\mathbf{w}^T(x-x'))^2|y=1, y'=-1]$ and  $\Omega$ be the convex regularizer where $p=Pr(y=+1)$ and $1-p=Pr(y=-1)$ are the class priors. We consider the following objective for surrogate-AUC maximization :
\begin{equation}
\begin{aligned}
    \min_{\mathbf{w} \in \mathbb{R}^d} f(\mathbf{w}) + \Omega(\mathbf{w})
    \label{eqn:for}
\end{aligned}
\end{equation}
 The form for $f(\mathbf{w})$ follows from the definition of AUC : expected pairwise loss between a positive instance and a negative instance. Throughout this paper we assume 
\begin{itemize}
    \item 
$\Omega$ is $\beta$ strongly convex i.e. for any $\mathbf{w},\mathbf{w}' \in \mathbb{R}^d, \Omega(\mathbf{w}) \geq \Omega(\mathbf{w}')+\partial \Omega(\mathbf{w}')^T(\mathbf{w}-\mathbf{w}')+\frac{\beta}{2}\|\mathbf{w}-\mathbf{w}'\|^2$ 
\item
$\exists M$ such that $\|x\| \leq M$  $\forall x \in \mathcal{X}$. 
\end{itemize}

In this paper we have used Frobenius norm $\Omega(\mathbf{w})=\beta\|\mathbf{w}\|^2$ and Elastic Net $\Omega(\mathbf{w})=\beta\|\mathbf{w}\|^2+\nu \|\mathbf{w}\|_1$ as the convex regularizers where $\beta, \nu \neq 0$ are the regularization parameters.

The minimization problem in equation \ref{eqn:for} can be reformulated such that stochastic gradient descent can be performed to find the optimum value. Below is an equivalent formulation from Theorem $1$ in \citet{natole2018stochastic}-
$$
\min_{\mathbf{w},a,b} \max_{\zeta \in \mathbb{R}}{\E[F(\mathbf{w},a,b,\zeta;z)]+\Omega(\mathbf{w})}
$$
where the expectation is with respect to $z=(x,y)$ and 
\begin{align*}
    F(\mathbf{w},&a,b,\zeta;z) = (1- p)(\mathbf{w}^Tx - a)^2\mathbbm{I}_{[y=1]}\\ &+ p(\mathbf{w}^Tx - b)^2\mathbbm{I}_{[y=-1]} + 2(1 + \zeta)\mathbf{w}^Tx(p\mathbbm{I}_{[y=-1]}\\& - (1 - p)\mathbbm{I}_{[y=1]}) - p(1 - p)\zeta^2 
\end{align*}
Thus, $f(\mathbf{w})=\min_{a,b} \max_{\zeta \in R}{\E[F(\mathbf{w},a,b,\zeta;z)]}$. \citet{natole2018stochastic} also state that the optimal choices for $a,b,\zeta$ satisfy : 
\begin{align*}
a(\mathbf{w})&=\mathbf{w}^T\E[x|y=1] \\ 
b(\mathbf{w})&=\mathbf{w}^T\E[x|y=-1] \\ \zeta(\mathbf{w})&=\mathbf{w}^T(\E[x'|y'=-1]-\E[x|y=1])
\end{align*}
An important thing to note here is that we differentiate the objective function only with respect to $\mathbf{w}$ and do not compute the gradient with respect to the other parameters which themselves depend on $\mathbf{w}$. This is the reason why existing methods cannot be applied directly.
\section{Method}
The major issue that slows down convergence for SGD is the decay of the step size to $0$ as the iteration increase. This is necessary for mitigating the effect of variance introduced by random sampling in SGD. We apply the Prox-SVRG method on the reformulation of AUC to derive the proximal SVRG algorithm for AUC maximization given in Algorithm \ref{alg:svrg}. We store a $\Tilde{\mathbf{w}}$ after every $m$ Prox-SGD iterations that is progressively closer to the optimal $\mathbf{w}$ (essentially an estimate of the optimal value of (\ref{eqn:for}). Full gradient $\Tilde{\bm{\mu}}$ is computed whenever $\Tilde{\mathbf{w}}$ gets updated---after every $m$ iterations of Prox-SGD:
\begin{align*}
    \Tilde{\bm{\mu}} = \frac{1}{n} \sum_{i=1}^{n} G(\Tilde{\mathbf{w}},z_{i})
\end{align*}
where  $G(\mathbf{w};z)=\partial_\mathbf{w} F(\mathbf{w},a(\mathbf{w}),b(\mathbf{w}),\zeta(\mathbf{w});z)$, $n$ is the number of samples and $\Tilde{\bm{\mu}}$ is used to update next $m$ gradients. Next $m$ iterations are initialized by $\mathbf{w}_0 = \Tilde{\mathbf{w}}$. For each iteration, we randomly pick $i_t \in \{ 1,...,n \}$ and compute 
\begin{align*}
\mathbf{\hat{w}}_t = \mathbf{w}_{t-1} - \eta \mathbf{v}_{t-1}    
\end{align*}
where $\mathbf{v}_t = G(\mathbf{w}_t,z_{i_{t-1}}) -G(\Tilde{\mathbf{w}},z_{i_{t-1}}) + \Tilde{\bm{\mu}}$ and then the proximal step is taken
\begin{align*}
    \mathbf{w}_t = \prox_{\eta,\Omega}(\mathbf{\hat{w}}_t)
\end{align*}
Notice that if we  take expectation of $G(\Tilde{\mathbf{w}},z_{i_{t-1}})$ with respect to $i_t$ we get $\E[G(\Tilde{\mathbf{w}},z_{i_{t-1}})]= \Tilde{\bm{\mu}}$. Now if we take expectation of $\mathbf{v}_t$ with respect to $i_t$ conditioned on $\mathbf{w}_{t-1}$, we can get the following:
\begin{align*}
    \E[\mathbf{v}_{t}| \mathbf{w}_{t-1}] & = \E[G(\mathbf{w},z_{i_{t-1}})] - \E[G(\Tilde{\mathbf{w}},z_{i_{t-1}})] + \Tilde{\bm{\mu}} \\
    & = \frac{1}{n} \sum_{i=1}^{n} G(\Tilde{\mathbf{w}}_{t-1},z_{i})
\end{align*}
Hence the modified direction $\mathbf{v}_{t}$ is stochastic gradient of $G$ at $\mathbf{w}_{t-1}$. However, the variance $\E\|\mathbf{v}_t-\partial f(\mathbf{w}_{t-1})\|^2$ can be much smaller than $\E\|G(\mathbf{w}_{t-1},z_{i_{t-1}})-\partial f(\mathbf{w}_{t-1})\|^2$ which we will show in section \ref{sub:var}. We will also show that the variance goes to 0 as the algorithm converges. Thus, this is a multi-stage scheme to explicitly reduce the variance of the modified proximal gradient.
 
%\begin{figure}[ttt!]
\begin{algorithm}[H]
\caption{Proximal SVRG for AUC maximization}\label{alg:svrg}
 \textsc{Input} Constant step size $\eta$ and update frequency $m$\\
  \textsc{Initialize} $\Tilde{\textsc{w}}_0$
\begin{algorithmic}
\For  {$s = 1,2,...$}
\State $\Tilde{\mathbf{w}} = \Tilde{\mathbf{w}}_{s-1}$
\State $\Tilde{\bm{\mu}} = \frac{1}{n} \sum_{i=1}^{n} G(\Tilde{\mathbf{w}},z_{i})$
\State $\mathbf{w}_0 = \Tilde{\mathbf{w}}$
\For {$t = 1,2,...,m$}
\State Randomly pick $i_t \in \{ 1,..,n \}$ and update weight
\State $\hat{\mathbf{w}}_t = \mathbf{w}_{t-1} - \eta (G(\mathbf{w}_{t-1},z_{i_{t}}) -G(\Tilde{\mathbf{w}},z_{i_{t}}) + \Tilde{\bm{\mu}})$
\State $\mathbf{w}_t = \prox_{\eta \Omega}(\hat{\mathbf{w}}_t)$
\EndFor
\State $\Tilde{\mathbf{w}}_s = \mathbf{w}_m$
\EndFor
\end{algorithmic}
\end{algorithm}
%\end{figure}

\section{Convergence Analysis}
In this section, we analyze the convergence rate of \textsc{VRSPAM} formally. We first define some lemmas which will be used for proving the Theorem \ref{thm1} which is the main theorem proving the geometric convergence of Algorithm \ref{alg:svrg}. First is the Lemma \ref{lemm:natole} from \citet{natole2018stochastic} which states that $\partial_{\mathbf{w}} F(\mathbf{w}_t,a(\mathbf{w}_t),b(\mathbf{w}_t),\alpha (\mathbf{w}_t);{z}_t)$ is an unbiased estimator of the true gradient. As we are not calculating the true gradient in \textsc{VRSPAM}, we need the following Lemma to prove the convergence result.
\begin{lemma}[\citep{natole2018stochastic}]
Let $\mathbf{w}_t$ be given by \textsc{VRSPAM} in Algorithm \ref{alg:svrg}. Then, we have
\begin{align*}
    \partial f(\mathbf{w}_t) = \E_{z_t} [\partial_{\mathbf{w}} F(\mathbf{w}_t,a(\mathbf{w}_t),b(\mathbf{w}_t),\alpha (\mathbf{w}_t);{z}_t)]
\end{align*}
\label{lemm:natole}
 \end{lemma}
 \vspace{-0.7em}
 This Lemma is directly applicable in \textsc{VRSPAM} since the proof of the Lemma hinges on the objective function formulation and not on the algorithm specifics. 

 The next lemma provides an upper bound on the norm of difference of gradients at different time steps. 
\begin{lemma}[\citep{natole2018stochastic}]
\label{lemma:bound}
Let $\mathbf{w}_t$ be described as above. Then, we have
\begin{align*}
   \|G(\mathbf{w}_{t'} ; z_{t}) - G(\mathbf{w}_t ; z_t) \| \leq 8M^2 \| \mathbf{w}_{t'} - \mathbf{w}_t \|
\end{align*}
\end{lemma}
\begin{proof}
\begin{equation*}
    \begin{aligned}
\|G(\mathbf{w}_{t'} ; z_{t}) &- G(\mathbf{w}_t ; z_t) \| \leq 4M^2p\|\mathbf{w}_{t'}-\mathbf{w}_t\|\mathbbm{1}_{[y_t=-1]}
\\&+
4M^2(1-p)\|\mathbf{w}_{t'}-\mathbf{w}_t\|\mathbbm{1}_{[y_t=1]}\\&+4M^2p\|\mathbf{w}_{t'}-\mathbf{w}_t\|\mathbbm{1}_{[y_t=-1]}\\&+4M^2|p-\mathbbm{1}_{[y_t=1]}|\|\mathbf{w}_{t'}-\mathbf{w}_t\|
\\& \leq 8M^2 \| \mathbf{w}_{t'} - \mathbf{w}_t \|
 \end{aligned}
 \end{equation*}
 The proof directly follows by writing out the difference and using the second assumption on the boundedness of $\|x\|$.
\end{proof}

We now present and prove a result that will be necessary in showing convergence in Theorem \ref{thm1}
\begin{lemma}
\label{lem:lma1}
Let $C = \frac{1 + 128M^4\eta^2}{(1 + \eta \beta)^2}$ and $D =\frac{128M^4\eta^2}{(1 + \eta \beta)^2}$; if $\eta\leq \frac{\beta}{128M^4}$ then $C^m  +  DC\frac{C^m-1}{C-1} \leq 1$ holds true.
\end{lemma}
%\noindent \textit{Proof.}
\begin{proof}
We start with:
    \begin{alignat*}{2}
    \eta & \leq \frac{\beta}{128M^4} \\
    \Rightarrow & 128M^4\eta^2 \leq \eta \beta \\
    \Rightarrow &128M^4\eta^2(2 + 128M^4\eta^2) \leq \eta \beta(2+1\eta \beta) \\
    \Rightarrow & 128M^4\eta^2 + (128M^4\eta^2)^2 \leq (\eta \beta)^2 + 2\eta \beta-128M^4\eta^2 \\
    \Rightarrow & 128M^4\eta^2 \leq \frac{(1+\eta \beta)^2 - 1-128M^4\eta^2}{1+128M^4\eta^2} \\
    \Rightarrow & 128M^4\eta^2 \leq \frac{1 - \frac{1+128M^4\eta^2}{(1+\eta \beta)^2}}{\frac{1+128M^4\eta^2}{(1+\eta \beta)^2}}
    \end{alignat*}
    Substituting values of $C$ and $D$ and using the condition that $D \leq 128M^4\eta^2$, we get
    \begin{align*}
    \Rightarrow & D \leq \frac{1-C}{C} \\
    \Rightarrow & DC\frac{C^m-1}{C-1}  \leq 1 - C^m \\
    \Rightarrow & C^m  +  DC\frac{C^m-1}{C-1}  \leq 1
\end{align*}
\end{proof}
The following is the main theorem of this paper stating the convergence rate of Algorithm 1 and its analysis.
\begin{theorem}
Consider \textsc{VRSPAM} (Algorithm \ref{alg:svrg}) and let $\mathbf{w^{*}} = \argmin_{\mathbf{w}} f(\mathbf{w}) + \Omega(\mathbf{w})$; if $\eta < \frac{\beta}{128M^4}$, then the following inequality holds true
\begin{align*}
    \alpha = C^m  +  DC\frac{C^m-1}{C-1} <  1
\end{align*}
and we have the geometric convergence in expectation:
\begin{align*}
\E [\|\mathbf{\tilde{w}}_{s}-\mathbf{w^{*}}\|^2] \leq \alpha^s \E [\|\mathbf{w_0}-\mathbf{w^{*}}\|^2]
\end{align*}
\label{thm1}
\end{theorem}
\vspace{-0.9em}
For proving the above theorem, first we upper bound the variance of the gradient step and show that it approaches zero as $\mathbf{w_s}$ approaches $\mathbf{w^{*}}$.
 \subsection{Bounding the variance}
 First we present a lemma what will be necessary to find the bound on the variance of modified gradient $\mathbf{v}_t =  G(\mathbf{w}_t,z_{i_t}) - G(\Tilde{\mathbf{w}},z_{i_t}) + \Tilde{\bm{\mu}} $
 \begin{lemma}
Consider \textsc{VRSPAM} (Algorithm \ref{alg:svrg}), then $\E [\| \mathbf{v}_t -  \partial f(\mathbf{w}^*)] \|^2$  is upper bounded as:
\begin{align*}
      \E [\| \mathbf{v}_t - & \partial f(\mathbf{w}^*)] \|^2] \\ & \leq 2(8M^2)^2\|\mathbf{w}_t - \mathbf{w}^* \|^2 + 2(8M^2)^2\|\Tilde{\mathbf{w}} - \mathbf{w}^* \|^2
\end{align*}
\label{thm2}
\end{lemma}
 \begin{proof}
 \label{sub:var}
 Let the variance reduced update be denoted as $\mathbf{v}_t =  G(\mathbf{w}_t,z_{i_t}) - G(\Tilde{\mathbf{w}},z_{i_t}) + \Tilde{\bm{\mu}} $. As we know $\E[\mathbf{v}_t]=\partial f(\mathbf{w}_t)$, the variance of $\mathbf{v}_k$ can be written as below
\begin{flalign*}
    \E [\| &G(\mathbf{w}_t,z_{i_t}) - G(\Tilde{\mathbf{w}},z_{i_t}) + \Tilde{\bm{\mu}} - \partial f(\mathbf{w}^*)) \|^2]   \\ & \leq  2\E [\| G(\mathbf{w}_t,z_{i_t}) - G({\mathbf{w}}^*,z_{i_t}) \|^2] \\ & + 2\E [\| G({\mathbf{w}}^*,z_{i_t}) - G(\Tilde{\mathbf{w}},z_{i_t}) + \Tilde{\bm{\mu}} - \partial f(\mathbf{w}^*)) \|^2]
\end{flalign*}
Also, $\E[ G({\mathbf{w}}^*,z_{i_t}) - G(\Tilde{\mathbf{w}},z_{i_t})  ] =\partial f(\mathbf{w}^*) -  \partial f(\Tilde{\mathbf{w}}) $ from Lemma \ref{lemm:natole} and using the property that $\E [(X - \E [X])^2] \leq \E[X^2]$ we get 
\begin{flalign*}
    \E [\| &G(\mathbf{w}_t,z_{i_t}) - G(\Tilde{\mathbf{w}},z_{i_t}) + \Tilde{\bm{\mu}} - \partial f(\mathbf{w}^*)) \|^2]  \\ &\leq 2\E [\| G(\mathbf{w}_t,z_{i_t}) - G({\mathbf{w}}^*,z_{i_t}) \|^2] \\& + 2\E [\| G({\mathbf{w}}^*,z_{i_t}) - G(\Tilde{\mathbf{w}},z_{i_t}) \|^2]
\end{flalign*}
From Lemma \ref{lemma:bound}, we have $\| G(\mathbf{w}_t,z_{i_t}) - G({\mathbf{w}}^*,z_{i_t}) \| \leq 8M^2\|\mathbf{w}_t - \mathbf{w}^* \|$ and $\| G({\mathbf{w}}^*,z_{i_t}) - G(\Tilde{\mathbf{w}},z_{i_t}) \|\leq 8M^2\|\Tilde{\mathbf{w}} - \mathbf{w}^* \|$. Using this, we can upper bound the variance of gradient step as:
\begin{equation}
\begin{aligned}
    \E [\| G(\mathbf{w}_t,z_{i_t}) - G(\Tilde{\mathbf{w}},z_{i_t}) + \Tilde{\bm{\mu}} - \partial f(\mathbf{w}^*)) \|^2]  \\ \leq 2(8M^2)^2\|\mathbf{w}_t - \mathbf{w}^* \|^2 + 2(8M^2)^2\|\Tilde{\mathbf{w}} - \mathbf{w}^* \|^2
\end{aligned}
\label{eqn:varbound}
\end{equation}
We have the desired result.
\end{proof}
We now present a lemma giving the bound on the variance of modified gradient $\mathbf{v}_t$
\begin{lemma}
Consider \textsc{VRSPAM} (Algorithm \ref{alg:svrg}), then the variance of the $\mathbf{v}_t$ is upper bounded as:
\begin{align*}
      \E [\| \mathbf{v}_t - & \partial f(\mathbf{w}_t)] \|^2] \big) \\ & \leq 4(8M^2)^2\|\mathbf{w}_t - \mathbf{w}^* \|^2 + 2(8M^2)^2\|\Tilde{\mathbf{w}} - \mathbf{w}^* \|^2
\end{align*}
\end{lemma}
\begin{proof}
\begin{align*}
     \E &[\| \mathbf{v}_t -  \partial f(\mathbf{w}_t)] \|^2]  \\ & \leq
     2\E [\| \mathbf{v}_t -  \partial f(\mathbf{w}^*)] \|^2] + 2\E [\| \partial f(\mathbf{w}^*) -  \partial f(\mathbf{w}_t)] \|^2]   \\ 
     & \leq 2(8M^2)^2\|\mathbf{w}_t - \mathbf{w}^* \|^2 + 2(8M^2)^2\|\Tilde{\mathbf{w}} - \mathbf{w}^* \|^2 \\ &+ 2\E[\| G(\mathbf{w}_t,z_{i_t}) - G({\mathbf{w}}^*,z_{i_t}) \|^2] \\
     & \leq 4(8M^2)^2\|\mathbf{w}_t - \mathbf{w}^* \|^2 + 2(8M^2)^2\|\Tilde{\mathbf{w}} - \mathbf{w}^* \|^2
\end{align*}
where the second inequality uses Lemma \ref{thm2} and last inequality uses Lemma \ref{lemma:bound}.
\end{proof}

At the convergence, $\Tilde{\mathbf{w}} = \mathbf{w}^*$ and $\mathbf{w}_t = \mathbf{w}^*$. Thus the variance of the updates are bounded and go to zero as the algorithm converges.
Whereas in the case of the SPAM algorithm, the variance of the gradient does not go to zero as it is a stochastic gradient descent based algorithm.

We now present the proof of Theorem \ref{thm1}.
\subsection{Proof of Theorem \ref{thm1}}

From the first order optimality condition, we can directly write
\begin{align*}
    \mathbf{w}^* = \prox_{\eta \Omega}(\mathbf{w}^* - \eta\partial f(\mathbf{w}^*))
\end{align*}
Using the above we can write
\begin{align*}
    &\|\mathbf{w}_{t+1} - \mathbf{w}^* \|^2 \\ &= \|\prox_{\eta \Omega}(\hat{\mathbf{w}}_{t+1}) - \prox_{\eta \Omega}(\mathbf{w}^* - \eta\partial f(\mathbf{w}^*))\|^2
\end{align*}
Using Proposition 23.11 from \citet{bauschke2011convex}, we have $\prox_{\eta \Omega}$ is  $(1 + \eta \beta)$-cocoercieve and for any $\mathbf{u}$ and $\mathbf{w}$ using Cauchy Schwartz we can get the following inequality 
\begin{align*}
    \|\prox_{\eta \Omega}({\mathbf{u}})-\prox_{\eta \Omega}({\mathbf{w}})\| \leq \frac{1}{1 + \eta \beta} \|\mathbf{u} - \mathbf{w} \|
\end{align*}
From above we get 
\begin{align*}
    &\|\mathbf{w}_{t+1} - \mathbf{w}^* \|^2 \\ & \leq \frac{1}{(1+\eta \beta)^2} \|(\hat{\mathbf{w}}_{t+1}) - (\mathbf{w}^* - \eta\partial f(\mathbf{w}^*))\|^2 \\& \leq \frac{1}{(1+\eta \beta)^2} \|(\mathbf{w}_t - \mathbf{w}^*)  - \eta(G(\mathbf{w}_t,z_{i_t}) \\& - G(\Tilde{\mathbf{w}},z_{i_t}) + \Tilde{\bm{\mu}} - \partial f(\mathbf{w}^*))\|^2 
\end{align*}
Taking expectation on both sides we get
\begin{equation}
    \begin{aligned}
\label{eq:bound}
     &\E \|\mathbf{w}_{t+1}
     - \mathbf{w}^* \|^2  \leq \frac{1}{(1 + \eta \beta)^2} \big (  \eta^2 \E [\| G(\mathbf{w}_t,z_{i_t }) \\& - G(\Tilde{\mathbf{w}},z_{i_t}) + \Tilde{\bm{\mu}} - \partial f(\mathbf{w}^*)) \|^2] + \E [\| \mathbf{w}_t - \mathbf{w}^*  \|^2] - \\ & 
     2\eta \E [\langle \mathbf{w}_t - \mathbf{w}^* , G(\mathbf{w}_t,z_{i_t}) - G(\Tilde{\mathbf{w}},z_{i_t}) + \Tilde{\bm{\mu}} - \partial f(\mathbf{w}^*) \rangle] \big )
\end{aligned}
\end{equation}
Now, we first bound the last term $T = \E [\langle \mathbf{w}_t - \mathbf{w}^* , G(\mathbf{w}_t,z_{i_t}) - G(\Tilde{\mathbf{w}},z_{i_t}) + \Tilde{\bm{\mu}} - \partial f(\mathbf{w}^*) \rangle]$ in equation \ref{eq:bound}. Using Lemma \ref{lemm:natole} we can write
\begin{align*}
    T & = \E [\langle \mathbf{w}_t - \mathbf{w}^* , \E_{z_t}[G(\mathbf{w}_t,z_{i_t})] - \E_{z_t}[G(\Tilde{\mathbf{w}},z_{i_t})] \\&+ \Tilde{\bm{\mu}} - \partial f(\mathbf{w}^*)\rangle] \\
    & = \E [\langle \mathbf{w}_t - \mathbf{w}^* , \E_{z_t}[G(\mathbf{w}_t,z_{i_t})]  - \partial f(\mathbf{w}^*)\rangle] \\
    & = \E [\langle \mathbf{w}_t - \mathbf{w}^* ,  \partial f(\mathbf{w}_t ) - \partial f(\mathbf{w}^*)\rangle] \\
    & \geq 0
\end{align*}
Now, $\E \|\mathbf{w}_{t+1} - \mathbf{w}^* \|^2$ can be bounded by using above bound and Lemma \ref{thm2} as below
\begin{align*}
   & \E \|\mathbf{w}_{t+1} - \mathbf{w}^* \|^2 \leq \frac{1}{(1 + \eta \beta)^2} (\E [\|\mathbf{w}_t - \mathbf{w}^* \|^2] \\ &+ 2(8M^2)^2\eta^2(\E [\|\mathbf{w}_t - \mathbf{w}^* \|^2]  +  \E [\|\Tilde{\mathbf{w}} - \mathbf{w}^* \|^2]  )) \\
    & \leq \frac{1 + 128M^4\eta^2}{(1 + \eta \beta)^2} \E [\|\mathbf{w}_t - \mathbf{w}^* \|^2] \\& + \frac{128M^4\eta^2}{(1 + \eta \beta)^2} \E [\|\Tilde{\mathbf{w}} - \mathbf{w}^* \|^2]  
\end{align*}
Let $C = \frac{1 + 128M^4\eta^2}{(1 + \eta \beta)^2}$ and $D =\frac{128M^4\eta^2}{(1 + \eta \beta)^2} $, then after $m$ iterations
$\mathbf{w}_{t} = \mathbf{\tilde{w}}_s$ and $\mathbf{w_{0}} = \mathbf{\tilde{w}}_{s-1}$. Substituting this in the above inequality, we get 
\begin{align*}
    &\E \|\mathbf{\tilde{w}}_{s} - \mathbf{\mathbf{w}^*} \|^2 \\ & \leq C^m \big( \E \|\mathbf{\tilde{w}}_{s-1} - \mathbf{w}^* \|^2 +  \sum_{i=0}^{m-1}\frac{D}{C^i} \E \|\mathbf{\tilde{w}}_{s-1} - \mathbf{w}^* \|^2 \big) \\
    & \leq   \big(C^m  +  \sum_{i=0}^{m-1}\frac{DC^m}{C^i}  \big)\E \|\mathbf{\tilde{w}}_{s-1} - \mathbf{w}^* \|^2 \\
    & \leq   \big(C^m  +  DC^m\frac{1-(1/C^m)}{1-(1/C)}  \big)\E \|\mathbf{\tilde{w}}_{s-1} - \mathbf{w}^* \|^2 \\
     & \leq   \big(C^m  +  DC\frac{C^m-1}{C-1}  \big)\E \|\mathbf{\tilde{w}}_{s-1} - \mathbf{w}^* \|^2 \\
     & \leq   \alpha \E \|\mathbf{\tilde{w}}_{s-1} - \mathbf{w}^* \|^2
\end{align*}
where $\alpha = C^m  +  DC\frac{C^m-1}{C-1}$ is the decay parameter, and $\alpha < 1$ by using Lemma \ref{lem:lma1}. After $s$ steps in outer loop of Algorithm \ref{alg:svrg}, we get $\E \|\mathbf{\tilde{w}}_{s} - \mathbf{w}^* \|^2 \leq \alpha^s \E \|\mathbf{w}_{0} - \mathbf{w}^* \|^2$ where $\alpha <1$. Hence, we get geometric convergence of $\alpha^s$ which is much stronger than the $\mathcal{O}(\frac{1}{t})$ convergence obtained in \citet{natole2018stochastic}. In the next section we derive the time complexity of the algorithm and investigate dependence of $\alpha$ on the problem parameters.
\subsection{Complexity analysis}
To get $\E \| \mathbf{\tilde{w}}_s - \mathbf{w}^*\|^2 \leq \epsilon$, number of iterations $s$ required is
\begin{align*}
    s \geq \frac{1}{\log \frac{1}{\alpha}} \log \frac{\E \| \mathbf{{w}}_0 - \mathbf{w}^*\|^2}{\epsilon}
\end{align*}
At each stage, the number of gradient evaluations are $n+2m$ where $n$ is the number of samples and $m$ is the iterations in the inner loop and the complexity is $\mathcal{O}(n + m)(\log(\frac{1}{\epsilon}))$ i.e. Algorithm \ref{alg:svrg} takes $\mathcal{O}(n + m)(\log(\frac{1}{\epsilon}))$ gradient complexity to achieve accuracy of $\epsilon$. Here, the complexity is dependent on $M$ and $\beta$ as $m$ itself is dependent on $M$ and $\beta$.

Now we find the dependence of $\alpha$ and $m$ on $M$ and $\beta$. Let $\eta = \frac{\theta \beta}{128 M^4}$ where $0 < \theta < 1$, then 
\begin{align*}
    C & =  \frac{1 + 128M^4\eta^2}{(1 + \eta \beta)^2} =  \frac{1 + \frac{\theta^2\beta^2}{128M^4}}{(1 + \frac{\theta\beta^2}{128M^4})^2} \\
    & < \frac{1 + \frac{\theta \beta^2}{128M^4}}{(1 + \frac{\theta\beta^2}{128M^4})^2} \\
    & = \frac{1}{(1 + \frac{\theta\beta^2}{128M^4})} \\
    & = E 
\end{align*}
therefore $D = \theta (E - E^2)$ and $DC < \theta E^2(1-E)$, using the above equations we can simplify $\alpha$ as
\begin{align*}
    \alpha & = C^m + DC \frac{1-C^m}{1-C} \\
    & < C^m + \theta E^2(1-E)\frac{1-C^m}{1-C} \\
    & < C^m + \theta E^2 (1-C^m) \quad \because \frac{1-E}{1-C} < 1 \\
    & = \theta E^2 + C^m -\theta E^2 C^m
\end{align*}
In the above equation, only $C^m -\theta E^2 C^m$ depends on $m$, if we choose $m$ to be sufficiently large then $\alpha = \theta E^2$. An important thing to note here is that $\theta E < C < E$, now if we choose $m \approx 2  \frac{\log \theta}{\log E}$ then $\alpha \approx 2\theta E^2$ which is independent of $m$. Thus the time complexity of the algorithm is $\mathcal{O}(n + 2 \frac{\log \theta}{\log E})(\log(\frac{1}{\epsilon}))$ when $m = \Theta(\frac{\log \theta}{\log E})$. As the order has inverse dependency on $\log E = \log \frac{128M^4}{128M^4 + \theta \beta^2}$, increase in $M$ will result in increase in number of iterations i.e. as the maximum norm of training samples is increased, larger $m$ is required to reach $\epsilon$ accuracy.

\begin{figure*}[h!btp]
    
       \begin{minipage}[t]{0.33\textwidth}
        \centering
                \includegraphics[trim={0cm 0.1cm 0cm 0.1cm},clip,width=0.99\textwidth]{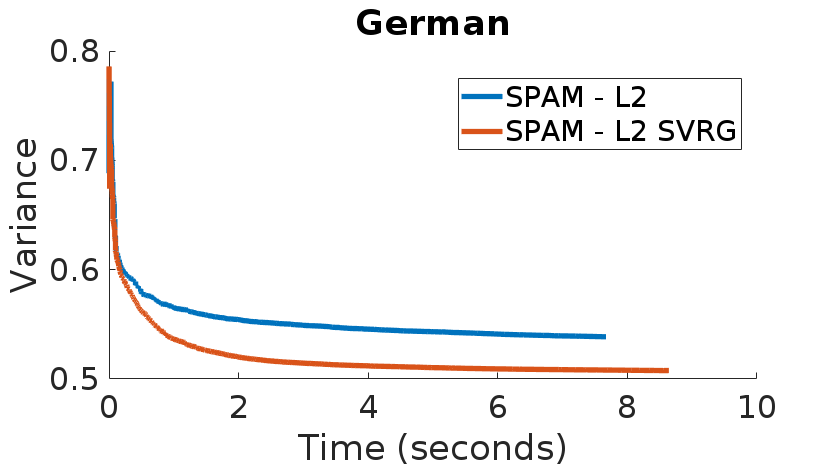}

    \end{minipage}
    \hfill
       \begin{minipage}[t]{0.33\textwidth}
        \centering
                \includegraphics[trim={0cm 0.1cm 0cm 0.1cm},clip,width=0.99\textwidth]{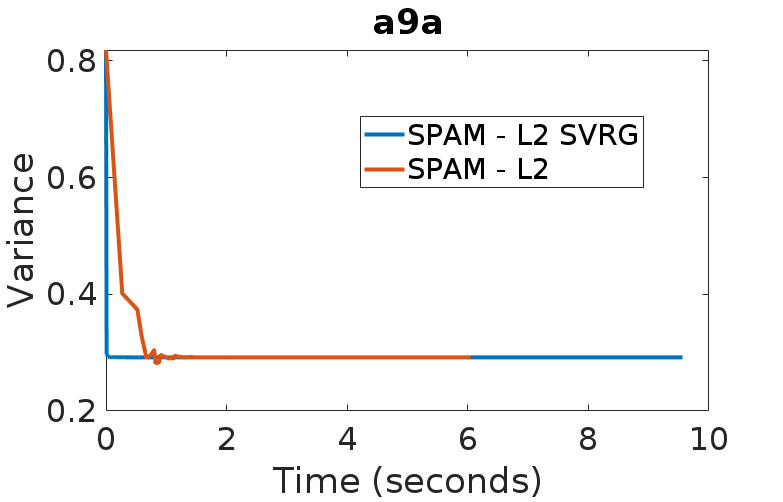}

    \end{minipage}
    \hfill
    \begin{minipage}[t]{0.33\textwidth}
        \centering
                \includegraphics[trim={0cm 0.1cm 0cm 0.1cm},clip,width=0.99\textwidth]{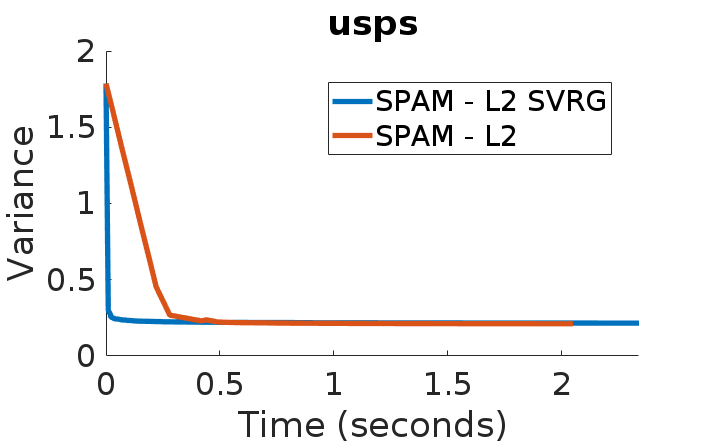}

    \end{minipage}
    
    \begin{minipage}[t]{0.33\textwidth}
        \centering
                \includegraphics[trim={0cm 0.1cm 0cm 0.1cm},clip,width=0.99\textwidth]{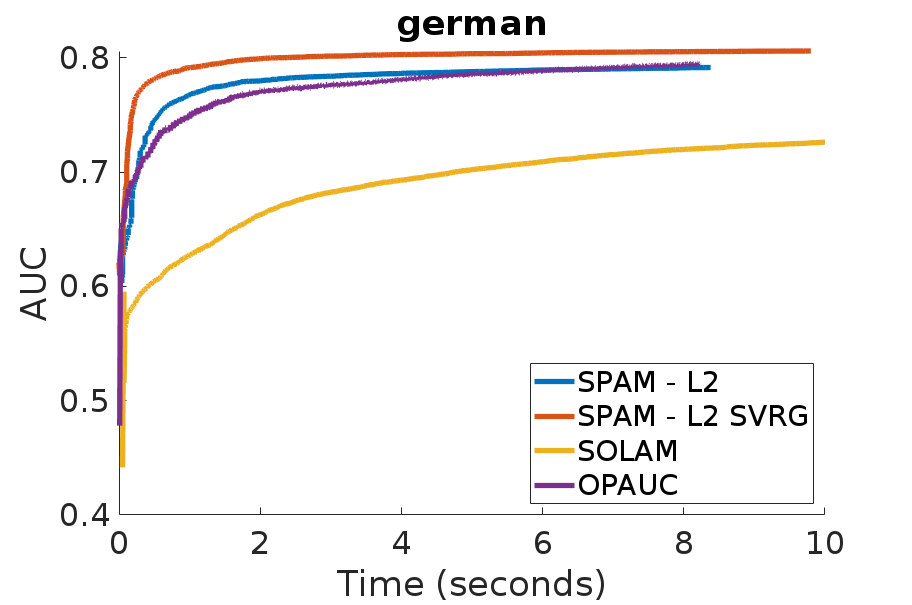}

    \end{minipage}
    \hfill
      \begin{minipage}[t]{0.33\textwidth}
        \centering
                \includegraphics[trim={0cm 0.1cm 0cm 0.1cm},clip,width=0.99\textwidth]{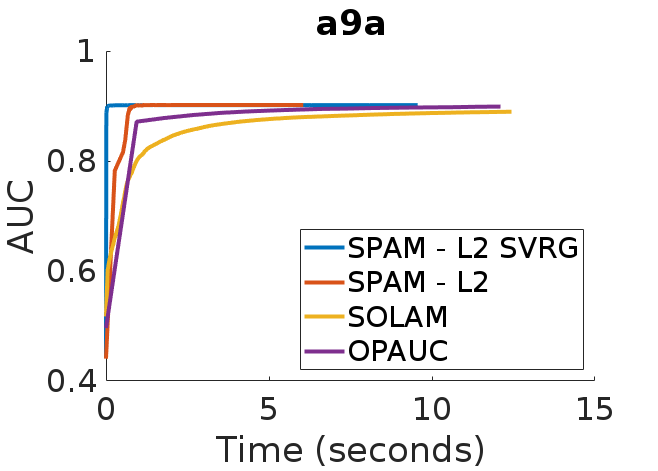}

    \end{minipage}
    \hfill
      \begin{minipage}[t]{0.33\textwidth}
        \centering
                \includegraphics[trim={0cm 0.1cm 0cm 0.1cm},clip,width=0.99\textwidth]{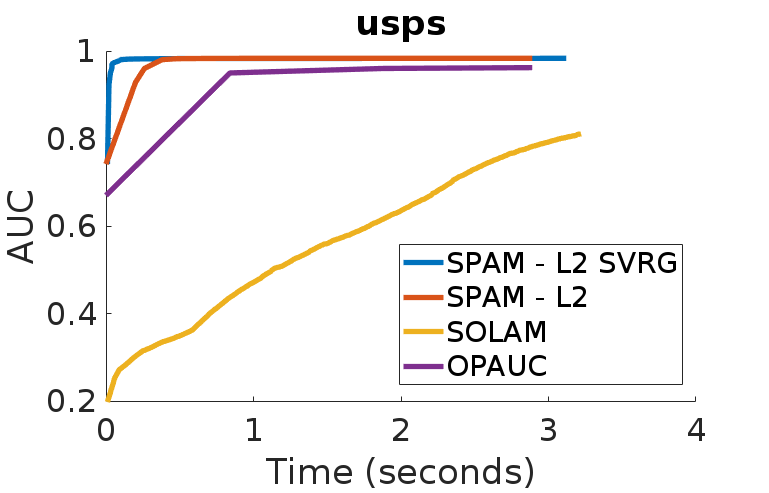}

    \end{minipage}

    %\caption{Example of 2-layer hierarchical structure}
        \caption{The top row shows that \textsc{VRSPAM} (SPAM-L2-SVRG) has lower variance than SPAM-L2 across different datasets. The bottom row shows \textsc{VRSPAM} (SPAM-L2-SVRG) converges faster and performs better than existing algorithms on AUC maximization}
    \label{fig:only}
    
\end{figure*}
\iffalse
\begin{figure*}[tt!]
    \centering
    \begin{subfigure}[t]{0.28\textwidth}
        \centering
        \includegraphics[width=\textwidth]{german_variance.png}
    \end{subfigure}%
    \begin{subfigure}[t]{0.28\textwidth}
        \centering
        \includegraphics[width=\textwidth]{german_main.png}
    \end{subfigure}
    \begin{subfigure}[t]{0.28\textwidth}
        \centering
        \includegraphics[width=\textwidth]{splice_variance.png}
    \end{subfigure}%
    ~
    \begin{subfigure}[t]{0.28\textwidth}
        \centering
        \includegraphics[width=\textwidth]{splice_convergence.png}
    \end{subfigure}
        \begin{subfigure}[t]{0.28\textwidth}
        \centering
        \includegraphics[width=\textwidth]{usps_variance.png}
    \end{subfigure}%
    \begin{subfigure}[t]{0.28\textwidth}
        \centering
        \includegraphics[width=\textwidth]{usps_main.png}
    \end{subfigure}
     \caption{The left column shows that \textsc{VRSPAM} (SPAM-L2-SVRG) has lower variance than SPAM-L2 across different datasets. The right column shows \textsc{VRSPAM} (SPAM-L2-SVRG) converges faster and performs better than existing algorithms on AUC maximization }
     \label{fig:only}
\end{figure*}
\fi
Now we will compare the time complexity of our algorithm with SPAM algorithm. First, we find the time complexity of SPAM. We will use Theorem 3 from \citet{natole2018stochastic} which states that SPAM achieves the following:
\begin{align*}
    \E [\| \mathbf{w}_{T+1}-&\mathbf{w}^* \|^2] \leq \frac{t_0}{T}\E [\| \mathbf{w}_{t_0}-\mathbf{w}^* \|^2] + c \frac{\log T}{T}
\end{align*}
where $t_0 = \text{max}\big( 2, \ceil[\big]{1 + \frac{(128M^4 + \beta ^2)^2}{128M^4\beta ^2}} \big)$, $T$ is the number of iterations and $c$ is a constant. Through averaging scheme developed by \citet{lacoste2012simpler} the following can be obtained:
\begin{align}
    \E [\| \mathbf{w}_{T+1}-&\mathbf{w}^* \|^2] \leq \frac{t_0}{T}\E [\| \mathbf{w}_{t_0}-\mathbf{w}^* \|^2] 
    \label{eqn:complx}
\end{align}
where $\E [\| \mathbf{w}_{t_0}-\mathbf{w}^* \|^2] \leq \frac{2\sigma_{*}^2}{\Tilde{C}^2_{\beta,M}} + \text{exp}\big( \frac{128M^4}{\Tilde{C}^2_{\beta,M}} \big) = F$, $\Tilde{C}^2_{\beta,M} = \frac{\beta}{(1+ \frac{\beta^2}{128M^4})^2}$ and $\E [\| G(\mathbf{w}^*;z)-\partial f(\mathbf{w}^*) \|^2] = \sigma_{*}^2$. Using equation \ref{eqn:complx}, time complexity of SPAM algorithm can be written as $\mathcal{O}(\frac{t_0 F}{\epsilon})$ i.e. SPAM algorithm takes $\mathcal{O}(\frac{t_0 F}{\epsilon})$ iterations to achieve $\epsilon$ accuracy. Thus, SPAM has lower per iteration complexity but slower convergence rate  as compared to VRSPAM. Therefore, VRSPAM will take less time to get a good approximation of the solution.

\begin{table*}[tt!]
\setlength{\tabcolsep}{0.31em}
  \centering
  \begin{tabular}{llllllllll}
    \toprule
    Name     & n & p   & \textsc{VRSPAM}-$L^2$     & \textsc{VRSPAM}-NET & SPAM-$L^2$ & SPAM-NET & SOLAM & OPAUC\\
    \midrule
    DIABETES  & 768 & 8 & .8299$\pm$.0323  & \textbf{.8305$\pm$.0319} & .8272$\pm$.0277 & .8085$\pm$.0431  & .8128$\pm$.0304 & .8309$\pm$.0350   \\
    %2     & .8325$\pm$.0375 & .8298$\pm$.0375  & .8210$\pm$.0203 & .8211$\pm$.0205 & .8213$\pm$.0209 & .8310$\pm$.0251     \\
    GERMAN       & 1000 & 24    & .7902$\pm$0386 & .7845$\pm$.0398 & .7942$\pm$.0388 & .7937$\pm$.0386 & .7778$\pm$.0373 & \textbf{.7978$\pm$.0347}  \\
    SPLICE  & 3,175 & 60    & .9640$\pm$.0156  & \textbf{.9699$\pm$.0139}  & .9263$\pm$.0091 & .9267$\pm$.0090 & .9246$\pm$.0087 & .9232$\pm$.0099   \\
    USPS & 9,298 & 256  & \textbf{.8552$\pm$.006} & .8549$\pm$.0059 & .8542$\pm$.0388 & .8537$\pm$.0386 & .8395$\pm$.0061 & .8114$\pm$.0065  \\
%     LETTER  & 20,000 & 16     & .9834$\pm$.0023 & .9804$\pm$.0032 & \textbf{.9868$\pm$.0032} & .9855$\pm$.0029 & .9822$\pm$.0036 & .9620$\pm$.0040     \\  
    A9A & 32,561   & 123    & .\textbf{9003$\pm$.0045}      & .8981$\pm$.0046 & .8998$\pm$.0046 & .8980$\pm$.0047 & .8966$\pm$.0043 & .9002$\pm$.0047 \\
    W8A & 64,700   & 300 & \textbf{.9876$\pm$.0008}   & .9787$\pm$.0013 &.9682$\pm$.0020 & .9604$\pm$.0020 & .9817$\pm$.0015  & .9633$\pm$.0035  \\
    MNIST & 60,000 & 780    & \textbf{.9465$\pm$.0014}  & .9351$\pm$.0014 & .9254$\pm$.0025 & .9132$\pm$.0026 & .9118$\pm$.0029 & .9242$\pm$.0021    \\
    ACOUSTIC & 78,823 & 50     & .8093$\pm$.0033  & .8052$\pm$.033 & .8120$\pm$.0030 & .8109$\pm$.0028 & 8099$\pm$.0036 & \textbf{.8192$\pm$.0032}     \\
    IJCNN1 & 141,691 & 22  & \textbf{.9750$\pm$.001}      & .9745$\pm$.002 & .9174$\pm$.0024 & .9155$\pm$.0024 & .9129$\pm$.0030 & .9269$\pm$.0021  \\
    \bottomrule
  \end{tabular}
  \caption{AUC values (mean$\pm$std) comparison for different algorithms on test data. $n$ denotes the number of training instances and $p$ denote the dimension of the feature space.}
    \label{table:result}
\end{table*}

\iffalse
\begin{figure*}[htt!]
    \centering
    \begin{subfigure}[t]{0.32\textwidth}
        \centering
        \includegraphics[width=\textwidth]{german_variance.png}
    \end{subfigure}%
    ~ 
    \begin{subfigure}[t]{0.32\textwidth}
        \centering
        \includegraphics[width=\textwidth]{usps_variance.png}
        
    \end{subfigure}
    ~
    \begin{subfigure}[t]{0.32\textwidth}
        \centering
        \includegraphics[width=\textwidth]{splice_variance.png}
        
    \end{subfigure}%
    
    \begin{subfigure}[t]{0.32\textwidth}
        \centering
        \includegraphics[width=\textwidth]{german_converg.png}
        
    \end{subfigure}
        \begin{subfigure}[t]{0.32\textwidth}
        \centering
        \includegraphics[width=\textwidth]{usps_convergence.png}
        
    \end{subfigure}%
    ~ 
    \begin{subfigure}[t]{0.32\textwidth}
        \centering
        \includegraphics[width=\textwidth]{splice_convergence.png}
        
    \end{subfigure}
     \caption{The left column shows that \textsc{VRSPAM} (SPAM-L2-SVRG) has lower variance than SPAM-L2 across different datasets. The right column shows \textsc{VRSPAM} (SPAM-L2-SVRG) converges faster and performs better than existing algorithms on AUC maximization }
     \label{fig:only}
\end{figure*}
\fi

\section{Experiment}

Here we empirically compare \textsc{VRSPAM} with other existing algorithms used for AUC maximization.
We use the following two variants of our proposed algorithm based on the regularizer used:
\begin{itemize}
    \item 
$\textsc{VRSPAM}-L^2 : \Omega(\mathbf{w}) = \frac{\beta}{2}\|\mathbf{w} \|^2$ (Frobenius Norm Regularizer) 
\item
$\textsc{VRSPAM}-NET : \Omega(\mathbf{w}) = \frac{\beta}{2}\|\mathbf{w} \|^2_2 + {\beta_1}\|\mathbf{w} \|_1$ (Elastic Net Regularizer \citep{zou2005regularization}).
The proximal step for elastic net is given as $\operatorname*{arg\,min}_{\mathbf{w}} 
 \lbrace \frac{1}{2}\|\mathbf{w}-\frac{\hat{\mathbf{w}}_{t+1}}{\eta_t \beta +1}\|^2+ \frac{\eta_t\beta_1}{\eta_t\beta+1}\|\mathbf{w}\|_1 \rbrace$

\end{itemize}

\textsc{VRSPAM} is compared with SPAM, SOLAM \citep{ying2016stochastic} and one-pass AUC optimization algorithm (OPAUC) \citep{gao2013one}. SOLAM was modified to have the Frobenius Norm Regularizer (as in \citet{natole2018stochastic}). \textsc{VRSPAM} is compared against OPAUC with the least square loss.

All datasets are publicly available from \citet{chang2011libsvm} and \citet{frank2010uci}. Some of the datasets are multiclass, and we convert them to binary labels by numbering the classes and assigning all the even labels to one class and all the odd labels to another. The results are the mean AUC score and standard deviation of 20 runs on each dataset. All the datasets were divided into training and test data with 80\% and 20\% of the data. The parameters $\beta \in 10^{[-5:5]}$ and $\beta_1 \in 10^{[-5:5]}$ for $\textsc{VRSPAM}-L^2$ and $\textsc{VRSPAM}-NET$ are chosen by 5 fold cross-validation on the training set. All the code is implemented in \textsc{matlab} and will be released upon publication. We measured the algorithm's computational time using an Intel i-7 CPU with a clock speed of 3538 MHz.
%For replication and fair comparison, we used the same strategy as mentioned in \cite{natole2018stochastic} to find the optimal parameter. 
\begin{itemize}

\item Variance results: In the top row of Figure \ref{fig:only}, we show the variance of the \textsc{VRSPAM} update ($\mathbf{v}_t$) in comparison with the variance of SPAM update ($G(\mathbf{w}_{t-1},z_{i_{t-1}})$) . We observe that the variance of \textsc{VRSPAM} is lower than the variance of SPAM and decreases to the minimum value faster, which is in line with Theorem {\ref{thm1}}.

    \item Convergence results: In the bottom row of Figure \ref{fig:only}, we show the performance of our algorithm compared to existing methods for AUC maximization. We observe that \textsc{VRSPAM} converges to the maximum value faster than the other methods, and in some cases, this maximum value itself is higher for \textsc{VRSPAM}.

\end{itemize}
Note that, the initial weights of \textsc{VRSPAM} are set to be the output generated by SPAM after one iteration, which is standard practice \citep{johnson2013accelerating}. Table\ref{table:result} summarizes the AUC evaluation for different algorithms. AUC values for SPAM-$L^2$, SPAM-NET, SOLAM and OPAUC were taken from \citet{natole2018stochastic}.

\section{Conclusion}
In this paper, we propose a variance reduced stochastic proximal algorithm for AUC maximization (\textsc{VRSPAM}). We theoretically analyze the proposed algorithm and derive a much faster convergence rate of $\mathcal{O}(\alpha^t)$ where $\alpha < 1$ (linear convergence rate), improving upon state-of-the-art methods \citet{natole2018stochastic} which have a convergence rate of $\mathcal{O}(\frac{1}{t})$ (sub-linear convergence rate), for strongly convex objective functions with per iteration complexity of one data-point. We gave a theoretical analysis of this and showed empirically \textsc{VRSPAM} converges faster than other methods for AUC maximization.

For future work, it will be interesting to explore if other algorithms to accelerate SGD can be used in this setting and if they lead to even faster convergence. It is also interesting to apply the proposed methods in practice to non-decomposable performance measures other than AUC. It would be interesting to extend the analysis to a non-convex and non-smooth regularizer using method presented in \citep{xu2019stochastic}.

\bibliography{ref}

\end{document}